\documentclass[conference]{IEEEtran}
\usepackage[letterpaper,
            top=0.7in,       
            bottom=1in,      
            left=0.625in,    
            right=0.625in,   
            columnsep=0.21in  
           ]{geometry}

\usepackage{amsthm,amsmath, amssymb, amsfonts}
\usepackage{graphicx}
\usepackage{float}
\usepackage[bookmarks=false]{hyperref} 
\usepackage{float}
\usepackage{enumitem}
\newtheorem{theorem}{Theorem}
\usepackage[ruled,vlined]{algorithm2e}
\usepackage{caption}
\usepackage{multirow}
\usepackage{subcaption}
\usepackage{xcolor}

\begin{document}

\title{Adaptive Budgeted Multi-Armed Bandits for IoT with Dynamic Resource Constraints}

\makeatletter
\newcommand{\linebreakand}{%
  \end{@IEEEauthorhalign}
  \hfill\mbox{}\par
  \mbox{}\hfill\begin{@IEEEauthorhalign}
}
\makeatother

\author{\IEEEauthorblockN{Shubham Vaishnav, Praveen Kumar Donta, and Sindri Magnússon}\\[-0.05em]
\IEEEauthorblockA{\textit{Department of Computer and Systems Sciences}, 
\textit{Stockholm University}, Stockholm 164 25, Sweden. \\
\{shubham.vaishnav, praveen.donta, sindri.magnusson\}@dsv.su.se}}

\maketitle

\begin{abstract}
Internet of Things (IoT) systems increasingly operate in environments where devices must respond in real time while managing fluctuating resource constraints, including energy and bandwidth. Yet, current approaches often fall short in addressing scenarios where operational constraints evolve over time. To address these limitations, we propose a novel Budgeted Multi-Armed Bandit framework tailored for IoT applications with dynamic operational limits. Our model introduces a decaying violation budget, which permits limited constraint violations early in the learning process and gradually enforces stricter compliance over time. We present the Budgeted Upper Confidence Bound (UCB) algorithm, which adaptively balances performance optimization and compliance with time-varying constraints. We provide theoretical guarantees showing that \texttt{Budgeted UCB} achieves sublinear regret and logarithmic constraint violations over the learning horizon. Extensive simulations in a wireless communication setting show that our approach achieves faster adaptation and better constraint satisfaction than standard online learning methods. These results highlight the framework’s potential for building adaptive, resource-aware IoT systems. 
\end{abstract}


\section{Introduction}

The number of Internet of Things (IoT) devices connecting through wireless networks is steadily increasing, thus facing tougher decision-making challenges in environments that are constantly changing and hard to predict.
In applications such as adaptive rate control, edge computing, and network resource allocation, agents must continuously select actions that optimize a primary performance objective (e.g., throughput, latency, or reliability) while simultaneously adhering to dynamic operational constraints (e.g., energy consumption, interference levels, or bandwidth budgets) \cite{xian2025robust, 10742099}.

A key challenge in these settings is the need to make decisions sequentially over time without full knowledge of the underlying system dynamics. Moreover, environmental constraints, such as energy thresholds or communication budgets, can vary over time due to changing network conditions, user demands, or hardware limitations. These realities necessitate learning frameworks that are capable of optimizing for multiple objectives under uncertainty and adapting to dynamically evolving constraint conditions.

Motivated by these challenges, we study a constrained stochastic bandit model with two objectives: a reward signal capturing the primary performance metric, and a constraint signal capturing secondary operational requirements. At each time step, the agent receives a constraint threshold from the environment, selects an action from a finite set, and observes stochastic feedback on both the reward and constraint signals. The goal is to maximize cumulative reward while ensuring that constraint violations are kept within a dynamically shrinking budget over time. This model is highly relevant for IoT systems that tighten operational tolerances over time (e.g., battery-draining IoT devices).

Our model is particularly suited to IoT and wireless communication scenarios, where real-time decision-making must be robust to dynamic resource constraints and evolving system demands. In the following sections, we formalize the model, define the optimization objective, and discuss strategies for achieving robust online learning under dynamic constraints.

\subsection{Related Work}
The problem of decision-making under uncertainty with resource constraints has been extensively studied in the domains of online learning and wireless communication. Classical multi-armed bandit (MAB) frameworks focus on maximizing cumulative rewards under stochastic feedback, with notable algorithms including UCB (Upper Confidence Bound) and Thompson Sampling~\cite{auer2002finite}. However, traditional MAB models do not account for operational constraints, limiting their applicability in dynamic IoT environments where resource availability fluctuates over time.

Recent advances have explored \textit{constrained multi-armed bandits} (CMAB), where the agent must optimize rewards while satisfying fixed constraints. Pioneering works such as Badanidiyuru \textit{et al.}~\cite{badanidiyuru2018bandits} proposed two algorithms for the Bandits with Knapsacks problem - \texttt{BalancedExploration} and \texttt{PrimalDualBwk} - that achieve near-optimal regret under fixed, known resource constraints. However, their framework assumes static budgets and does not address scenarios with dynamically evolving constraints.

Another line of research investigates \textit{safe exploration} techniques, such as Safe-UCB by Sui \textit{et al.}~\cite{sui2015safe} and Moradipari \textit{et al.}~\cite{moradipari2020stage}. These approaches typically assume static constraints or conservative assumptions that limit reward optimization.

\textit{Multi-objective bandits} approaches, such as Drugan and Nowé~\cite{drugan2013designing} extend bandit models to handle multiple competing reward dimensions, optimizing for trade-offs via Pareto-based or scalarization methods. These methods treat all objectives symmetrically and do not differentiate between a primary performance objective (e.g., throughput) and secondary operational constraints (e.g., energy use). In contrast, a more recent work by Wang et al.~\cite{wang2025neural} integrates constrained optimization for safe online learning.

Recent studies-such as Chen and Giannakis~\cite{chen2018bandit} for energy-critical IoT systems and Liu and Fang~\cite{10228901} for 6G IoT task scheduling—demonstrate the value of incorporating constraints into online learning and decision-making. While these models address important operational limits, they typically assume perfect or predictable knowledge of constraints and thus lack the flexibility needed to handle stochastic, time-varying budgets and requirements.

While Neely and Yu’s OCO algorithm \cite{neely2017online} develops a virtual-queue method for full-information convex optimization with time-varying constraints under i.i.d.\ assumptions, our work addresses the \emph{partial-information} MAB setting with stochastic reward and cost feedback. Cao and Liu \cite{cao2018time, cao2018online} study online convex optimization with time-varying constraints, addressing both full information and bandit feedback settings. Their algorithms achieve sublinear regret and constraint violation, assuming sublinear drift of the comparator sequence. 

Reinforcement learning is a promising approach for these dynamically changing and constrained 6G IoT environments, due to its ability to learn optimal policies while interacting with the environment \cite{Vaishnav2024}. We introduce a \emph{decaying violation budget} and propose a \texttt{Budgeted UCB} algorithm that allows controlled exploration-phase violations yet enforces vanishing constraint breaches over time, obtaining novel convergence guarantees in the dynamic-IoT bandit context. 

\subsection{Contributions}
This work introduces a novel stochastic bandit model and \texttt{budgeted UCB} algorithm designed for dynamic constraint satisfaction in real-time decision-making environments. In contrast to traditional constrained bandit formulations that focus on cumulative or static constraint management, we propose a per-round violation model with a dynamically shrinking budget, directly motivated by IoT and wireless communication applications where operational thresholds evolve over time. To the best of our knowledge, this is the first stochastic bandit model that explicitly:
\begin{itemize}
    \item Allows controlled constraint violations during initial learning phases, with a dynamically shrinking violation budget.
    \item Adapts exploration and exploitation strategies based on real-time constraint satisfaction metrics.
    \item Provides theoretical guarantees on both sublinear regret and logarithmic constraint violation rates in dynamically constrained environments.
    \item Models real-world IoT decision-making environments where operational thresholds are not static but evolve based on system states and external conditions.
\end{itemize}

Thus, our work bridges an important gap between theoretical constrained bandits and practical IoT applications requiring adaptive, resource-efficient learning mechanisms.


\section{Problem Formulation}

We consider a stochastic multi-armed bandit problem with dynamic constraints where an agent must simultaneously maximize a reward signal while adhering to evolving constraints over time. Such settings arise naturally in applications including, but not limited to, online resource allocation, recommendation systems, and energy management.

Let $\mathcal{A} = \{a_1, \dots, a_K\}$ denote a finite set of $K$ actions (arms). The agent interacts with the environment over a finite horizon of $T$ discrete time steps. At each time step $t = 1, \dots, T$:
\begin{enumerate}
    \item The environment issues a \emph{constraint threshold} $C_t \in \mathbb{R}$.
    \item The agent selects an \emph{action} $a_t \in \mathcal{A}$.
    \item The environment returns a \emph{stochastic feedback pair} $(r_t, c_t)$, where:
    \begin{itemize}
        \item $r_t \in \mathbb{R}$ is the \emph{reward signal},
        \item $c_t \in \mathbb{R}$ is the \emph{observed constraint signal}.
    \end{itemize}
\end{enumerate}
The feedback $(r_t, c_t)$ is drawn according to an action-dependent distribution $\mathcal{D}_{a_t}$, i.e., $
    (r_t, c_t) \sim \mathcal{D}_{a_t},$ where for each $i \in \{1,2\}$, the conditional expectations are given by
\[
    \mu_r(a)=\mathbb{E}[r_t \mid a_t = a] , \quad \mu_c(a)=\mathbb{E}[c_t \mid a_t = a] ,
\]
with $\mu_r(a)$ denoting the expected reward and $\mu_c(a)$ denoting the expected constraint feedback associated with action $a$.
Throughout, we distinguish between the \emph{issued constraint threshold} $C_t$ that is externally imposed at time $t$, and the \emph{observed constraint feedback} $c_t$ resulting from the agent's chosen action. 

The agent's objective is to select a sequence of actions $\{a_t\}_{t=1}^T$ to maximize the expected cumulative reward, while ensuring that the reward signal satisfies the dynamic constraint thresholds issued by the environment.
We consider the constraint to be satisfied at time $t$ if the expected reward of the chosen action meets or exceeds the issued threshold, i.e., $\mu(a_t) \geq c_t.$

To formalize constraint satisfaction, we introduce a dynamically shrinking per-round violation budget. Specifically, we define:
\begin{align}
    \delta_t &= \delta_0 \left( 1 - \frac{t-1}{T_{bud}} \right), \quad 0 < \delta_0 < 1, \label{eq:delta_t} \\
    I_t &= \mathbf{1}\{ c_{t} > C_t \}, \\
    v_t &= \frac{1}{t} \sum_{s=1}^t I_s, \label{eq:violation_rate}
\end{align}
where:
\begin{itemize}
    \item $\delta_t$ denotes the permissible violation rate (budget) at time $t$, starting from an initial allowance $\delta_0$ and decreasing linearly to zero over a duration of $T_{bud} \le T$,
    \item $I_t$ is the indicator variable that equals one if a constraint violation occurs at time $t$,
    \item $v_t$ is the empirical violation rate up to time $t$.
\end{itemize}

Thus, constraint satisfaction requires the agent to keep its cumulative violation rate $v_t$ below the shrinking budget $\delta_t$ at each time step, counting a violation whenever $c_t > C_t$.


The agent seeks to design a policy $\pi$ that selects actions $\{a_t\}_{t=1}^T$ so as to maximize the expected cumulative reward, while ensuring compliance with the violation budget. Formally, the objective is:
\begin{align}
    \max_{\pi} \quad & \mathbb{E}_\pi\left[ \sum_{t=1}^T r_{t} \right], \label{eq:obj_throughput} \\
    \text{subject to} \quad & \mathbb{E}_\pi[v_t] \leq \delta_t, \quad \forall t \in \{1, \dots, T\}, \label{eq:budget_constraint}
\end{align}
where the expectation is taken over the randomness of the policy $\pi$ and the environment.

The agent does not have prior knowledge of the expected reward functions $\mu_r$ and $\mu_c$, nor the distributions $\mathcal{D}_{a}$. The constraint thresholds $\{C_t\}_{t=1}^T$ are observed at the beginning of each round. The agent must learn from interaction with the environment and adapt its policy online.

\section{\texttt{Budgeted UCB} Algorithm}

    We now present the Budgeted Upper Confidence Bound (\texttt{Budgeted UCB}) algorithm, a method designed to address the exploration-exploitation trade-off under a budget constraint. \texttt{Budgeted UCB} extends the classical UCB strategy by accounting for both the expected rewards and the associated costs of actions. At each decision round, the algorithm selects the arm that optimizes a reward-to-cost adjusted upper confidence index, ensuring efficient budget utilization while maintaining strong performance guarantees.


The \texttt{Budgeted UCB} algorithm initializes, for each arm \(a\), a play count \(N(a)=0\) and cumulative reward sums \(S_r(a)=S_c(a)=0\), respectively, for the reward signals $r_t$ and constrain signals $c_t$. At each round \(t\), it first observes the current constraint threshold \(C_t\). Then, for each arm $a\in \mathcal{A}$ it computes the two upper confidence bounds
\[
\mathrm{UCB}_i(a)
=
\frac{S_i(a)}{N(a)}
+
\sqrt{\frac{2\ln t}{N(a)}},
\qquad
i=r,c,
\]
treating the arms with \(N(a)=0\) as infinitely optimistic. Next, it updates the linearly decaying violation allowance
\[
\delta_t
=
\delta_0\Bigl(1 - \frac{t-1}{T_{bud}}\Bigr),
\]
where $T_{bud}$ is the duration for which the violation allowance budget is non-zero. The empirical violation rate is then computed as:
\[
v_t
=
\frac{1}{t-1}\sum_{s=1}^{t-1}\mathbf1\{c_{t} > C_t\}.
\]
If \(v_t \le \delta_t\), the algorithm remains in its “exploration” phase and selects the arm with the highest throughput UCB, \(\arg\max_a \mathrm{UCB}_r(a)\). Otherwise it enters “safety” mode: it forms the feasible set
\[
\mathcal{F}_t = \{\,a : \mathrm{UCB}_c(a)\le c_t\},
\]
and if \(\mathcal{F}_t\neq\emptyset\) selects \(\arg\max_{a\in\mathcal{F}_t}\mathrm{UCB}_r(a)\); if no arm looks safe, it picks \(\arg\min_a\mathrm{UCB}_c(a)\) to minimize further violations. Finally, after playing arm \(a_t\), it observes the rewards \((r_{t},c_{t})\), increments \(N(a_t)\), and updates \(S_r(a_t)\leftarrow S_r(a_t)+r_{t}\) \(S_c(a_t)\leftarrow S_c(a_t)+c_{t}\) before proceeding to the next round. 

\begin{algorithm}[t] 
\caption{\texttt{Budgeted UCB} with Decaying Violation Budget}
\SetKwInOut{Input}{Input}\SetKwInOut{Output}{Output}
\Input{Arms $\mathcal{A}=\{a_1,\dots,a_K\}$, horizon $T$, initial violation rate $\delta_0$}
\Output{Action sequence $A_1,\dots,A_T$}
\BlankLine
\ForEach{$a\in\mathcal{A}$}{
  $N(a)\leftarrow 0$, \quad $S_1(a)\leftarrow 0$, \quad $S_2(a)\leftarrow 0$\;
}
\For{$t\leftarrow 1$ \KwTo $T$}{
  Observe constraint $C_t$\; \par
\ForEach{$a\in\mathcal{A}$}{
    \uIf{$N(a)>0$}{
      $\hat\mu_i(a)\leftarrow S_i(a)/N(a)~~~~~~~~~~~~~~~~~~~~\text{for }i=r,c$\;\par
      $\mathrm{UCB}_i(a)\leftarrow \hat\mu_i(a) + \sqrt{2\ln t \,/\, N(a)}~~\text{for }i=r,c$\;
    }
    \Else{
      $\mathrm{UCB}_r(a)\leftarrow +\infty$, \quad $\mathrm{UCB}_c(a)\leftarrow +\infty$\;
    }
  }
  $\;\delta_t \leftarrow \delta_0 \bigl(1 - (t-1)/T\bigr)$ \tcp*{decaying budget}
  $\;
  \text{Calculate $v_t$ using eq. \eqref{eq:violation_rate}}$\;\par
  \If{$v_t \le \delta_t$}{
    $a_t \leftarrow \arg\max_{a\in\mathcal{A}}\mathrm{UCB}_r(a)$ \tcp*{explore for throughput}
  }{
    $\mathcal{F}\leftarrow\{\,a:\mathrm{UCB}_c(a)\le c_t\}$\; \par
    \If{$\mathcal{F}\neq\emptyset$}{
      $a_t \leftarrow \arg\max_{a\in\mathcal{F}}\mathrm{UCB}_r(a)$ \tcp*{safe explore}
    }
    \Else{
      $a_t \leftarrow \arg\min_{a\in\mathcal{A}}\mathrm{UCB}_c(a)$ \tcp*{minimize violation}
    }
  }
  Play $a_t$, observe $(r_{t},c_{t})$\; \par
  $N(a_t)\!\leftarrow N(a_t)+1$, \quad $S_i(a_t)\!\leftarrow S_i(a_t)+r_{t,i}$\;\par
}
\end{algorithm}

In summary, the \texttt{Budgeted UCB} algorithm maintains upper‐confidence bounds for both objectives.  Early on, it allows a fraction $\delta_t$ of violations (to encourage exploration of high‐throughput arms).  As $t$ increases, $\delta_t$ decays linearly to zero.  If the observed violation rate $v_t$ stays within budget, the algorithm purely maximizes throughput UCB.  Otherwise, it switches to a safety‐first policy: filter out arms whose energy‐UCB exceeds $c_t$, then pick the best throughput among safe arms, or if none, the arm least likely to violate. Assuming at least one safe arm always exists, the \texttt{Budgeted UCB} algorithm implicitly guarantees the solution of the optimization problem \ref{eq:obj_throughput} under the constraint \ref{eq:budget_constraint}.

\section{Theoretical Results}
We now analyze the performance of the \texttt{Budgeted UCB} algorithm. Specifically, we establish upper bounds on the regret with respect to the optimal budget-respecting policy. Our analysis shows that \texttt{Budgeted UCB} achieves sublinear regret while ensuring that constraint violations vanish asymptotically.

We first introduce the necessary notation and formalize the notion of regret in the budgeted setting. The cumulative regret and violations are defined as follow.
\[
R(T)
=
\sum_{t=1}^T \Delta_t,
\qquad
V(T)
=
\sum_{t=1}^T \mathbf{1}\{\mu_2(a_t)>C_t\}.
\]
We now establish our regret bounds.

\begin{theorem}[Regret and Violation Bounds]
With probability at least $1-1/T$, \texttt{Budgeted UCB} satisfies
\[
R(T)
= O\bigl(\sqrt{K\,T\ln T}\bigr),
\qquad
V(T)
= O(\ln T).
\]
\end{theorem}

\begin{proof} We split the regret analysis into three parts: we first bound the regret incurred during rounds where the algorithm explores; second, we bound the regret from rounds where the algorithm plays safe actions; and third, we combine these two bounds to obtain the total regret.

\medskip\noindent\textbf{Part 1 - Exploration Regret Bound.}
As long as the observed violation rate
\[
v_t \;=\; \frac{\sum_{s=1}^{t-1}\mathbf1\{\hat\mu_c(a_s)>C_s\}}{\,t-1\,}
\;\le\;\delta_t,
\]
the algorithm ignores the constraint and plays $a_t = \arg\max_a \mathrm{UCB}_r(a).$

The exact regret and violation bounds for the exploration phase of \texttt{Budgeted UCB} are derived from the classical analysis of the Upper Confidence Bound (UCB) algorithm found in the foundational work of Auer et al. (2002)~\cite{auer2002finite}. Specifically, with probability $1-1/T$, the bound on regret in the exploration phase is given explicitly as $R_{exp}(T)$:
\[
R_{exp}(T) \leq 8 \sum_{a:\mu(a)<\mu^*} \frac{\ln T}{\Delta(a)} + \left(1+\frac{\pi^2}{3}\right)\sum_{a=1}^K\Delta(a),
\]
Simplifying this, we get:
\begin{align}
R_{\mathrm{exp}}(T)
&= O\bigl(\sqrt{K\,T\ln T}\bigr).
\end{align}

During the exploration phase, violations are explicitly allowed by the algorithmic design, controlled by the initial violation budget parameter $\delta_0$:
\begin{align}
    V_{\mathrm{exp}}(T) = O(T_{bud} \delta_0).
\end{align}
Note that the violation budget, $\delta_t$, reduces to $0$ after a constant $T_{bud}$ exploration period.

\medskip\noindent\textbf{Part 2 - Safety Regret Bound.}
Once $v_t > \delta_t$, the algorithm restricts to $\mathcal{F}_t=\bigl\{a: \mathrm{UCB}_c(a)\le c_t \bigr\}$.
If $\mathcal{F}_t\neq\emptyset$, it picks
\[
a_t = \arg\max_{a\in\mathcal{F}_t} \mathrm{UCB}_r(a),
\]
else it chooses $a_t = \arg\min_a \mathrm{UCB}_c(a).$

Violations in this phase occur only if the algorithm mistakenly believes that an arm's energy consumption is below the constraint when it is actually above it. Mathematically, a violation in this phase can only occur if
\(\mu_c(a_t)>c_t\) while \(\mathrm{UCB}_c(a_t)\le c_t\).  According to Hoeffding’s inequality, which gives the upper bounds for such randomness, each arm contributes at most a logarithmic number of violations ($O\!\Bigl(\ln T\Bigr)$). Summing over all $K$ arms gives
\[
V_{\mathrm{safe}}(T)
= O\bigl(K\ln T\bigr)
= O(\ln T).
\]
Meanwhile, because the best feasible arm $a^*_t$ is always in $\mathcal{F}_t$ (except on these $O(\ln T)$ errors), the regret in this phase also satisfies
\[
R_{\mathrm{safe}}(T)
= O\bigl(\sqrt{K\,T\ln T}\bigr).
\]

\medskip\noindent\textbf{Part 3 - Total Regret Bound.}
With probability at least $1-1/T$,
\begin{align*}
R(T)
&= R_{\mathrm{exp}}(T) + R_{\mathrm{safe}}(T)
= O\bigl(\sqrt{K\,T\ln T}\bigr),
\\
V(T)
&= V_{\mathrm{exp}}(T) + V_{\mathrm{safe}}(T)
= O(T_{bud} \delta_0 + \ln T) = O(\ln T).
\end{align*}
\end{proof}
To understand the implications of the bound, we note that the influence of the term $T_{bud} \delta_0$ on $V(T)$ shrinks to zero as the exploration ends. Thus, the violations quickly stabilize to only logarithmic growth. Moreover, since $R(T)=o(T)$ and $V(T)=o(T)$, we have
\[
\frac{R(T)}{T}\;\to\;0,
\qquad
\frac{V(T)}{T}\;\to\;0,
\]
i.e.\ both the average regret and violation rates vanish as $T\to\infty$. Thus, the \texttt{Budgeted UCB} algorithm achieves the optimal $O(\sqrt{T\ln T})$ throughput regret of standard UCB while incurring only $O(\ln T)$ constraint violations, and guarantees both average regret and violation rate converge to zero.

\section{Experimental Results}

\subsection{Simulation Setup and Baselines}

\begin{figure}[b]
    \begin{subfigure}[b]{0.23\textwidth} 
         \centering \includegraphics[width=\textwidth]{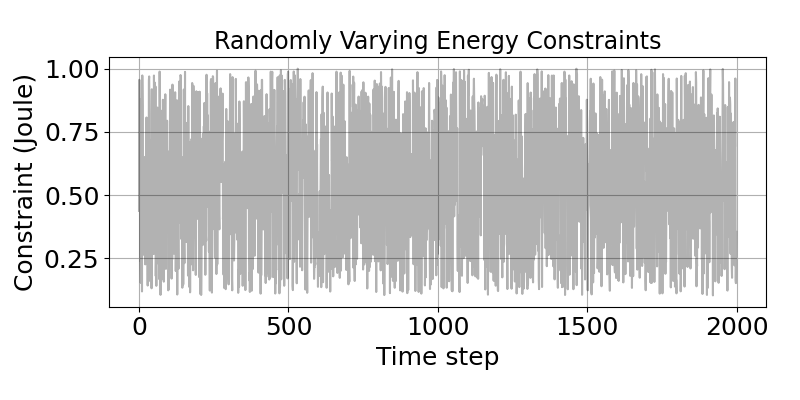}
    \caption{Exp. 1: Randomly varying energy constraints}
  \end{subfigure}
  \hfill
  \begin{subfigure}[b]{0.23\textwidth} 
         \centering
\includegraphics[width=\textwidth]{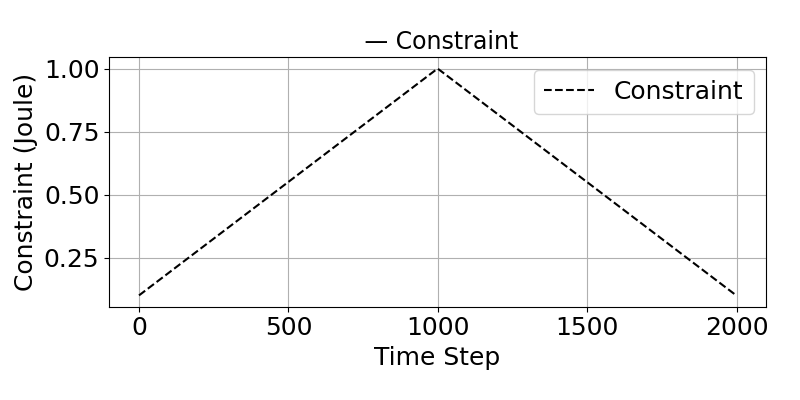}
    \caption{Exp. 2: Linearly varying energy constraints}
  \end{subfigure}
  \caption{Energy Constraint variations in the two experiments}
  \label{fig: constraints}
\end{figure}

We consider a scenario of an IoT device, a battery operated transmitter which sends data over a wireless channel to a fixed receiver located \(d=10\)m away. Our aim is to maximize the \emph{cumulative throughput} $\sum_{t=1}^T r_t$ 
over a horizon of \(T=2000\) time steps, while keeping the \emph{empirical energy‐constraint violation rate $v_t$ below a linearly decaying budget $\delta_t$}: 
\[
v_t \;=\; \frac{1}{t}\sum_{s=1}^t \mathbf1\{c_s > C_s\},
\qquad
\delta_t \;=\; \delta_0 \Bigl(1 - \tfrac{t-1}{T}\Bigr)
\]
where \(\delta_0=0.5\). At each round an energy constraint \(C_t\) is imposed, as shown in Figure \ref{fig: constraints}. We conduct two experiments: in Experiment 1, \(C_t\) is drawn uniformly in \([P_{\min},P_{\max}]\), and in Experiment 2 drifting linearly down from \(P_{\max}\) to \(P_{\min}\) then back up.  Our \texttt{Budgeted UCB} algorithm uses \(\delta_t\) to decide when to explore purely for throughput versus when to filter arms by their cost‐UCB to avoid violations.

The wireless communication channel has bandwidth \(B=1\)MHz and experiences path‐loss \(g=d^{-\alpha}\) with exponent \(\alpha=3\), plus additive white Gaussian noise of spectral density \(N_0=10^{-9}\,\mathrm{W/Hz}\). To adapt to changing energy constraints, the transmitter may choose at each time \(t\) one of \(K\) discrete power levels $P(a)\;\in\;\{P_{\min},\,P_{\min}+\Delta P,\;\dots,\;P_{\max}\}, $
where \(P_{\min}=0.1\)\,W, \(P_{\max}=1.0\)\,W and \(\Delta P=(P_{\max}-P_{\min})/(K-1)\).  Transmitting at power \(P(a_t)\) yields an instantaneous data rate (throughput) and energy cost
\[
r_t \;=\; B\log_2\!\Bigl(1 + \tfrac{P(a_t)\,g}{N_0\,B}\Bigr),
\qquad
e_t \;=\; P(a_t).
\]

We compare the performance of the proposed algorithm with three \textbf{baseline policies} over $T=2000$ steps in a wireless link simulation with randomly varying energy constraints:
\begin{itemize}
  \item \textbf{\texttt{Unconstrained UCB (u1):}} Classic UCB algorithm applied solely to throughput.
  \item \textbf{\texttt{Thompson Sampling (ts):}} Samples the throughput of each arm from a Gaussian posterior and selects the arm with the highest sampled value.
  \item \textbf{\texttt{Epsilon–Greedy (eg):}} With probability $\varepsilon$ explores uniformly at random, otherwise it exploits the empirically best‐throughput arm.
  \item \textbf{\texttt{Virtual Queue (vq) \cite{neely2017online}:}} Implements the method from “Time‐Varying Constraints and Bandit Feedback in Online Convex Optimization,” using a virtual queue to enforce dynamic constraints by penalizing arms according to accumulated violations.

\end{itemize}

The experiments were conducted on a machine equipped with an 11th Gen Intel(R) Core(TM) i7-1185G7 CPU @ 3.00GHz, 32~GB RAM, and running Windows. Each experiment was repeated $5$ times, and the results were averaged to ensure statistical robustness. Key simulation parameters are summarized in Table~\ref{Tab:sim_params}.


\begin{table}[h!]
\centering
\caption{Simulation Parameters}
\label{Tab:sim_params}
\begin{tabular}{|p{2.55cm}|p{1.25cm}|p{2.35cm}|p{1.2cm}|}
 \hline
 \multicolumn{4}{|c|}{\textbf{Simulation Parameters}} \\
 \hline
 \textbf{Parameter} & \textbf{Value} & \textbf{Parameter} & \textbf{Value} \\
 \hline
 Horizon $T$ & $2000$ steps & Distance $d$ & $10$ m \\
 Bandwidth $B$ & $1$ MHz & Noise Density $N_0$ & $10^{-9}$ W/Hz \\
 Path-loss exponent $\alpha$ & $3$ & Number of Arms & $11$ \\
 Min. Power $P_{\text{min}}$ & $0.1$ W & Max. Power $P_{\text{max}}$ & $1.0$ W \\
 \hline
\end{tabular}
\end{table}

\begin{figure*}[!t]
  \centering
  \begin{subfigure}{0.32\textwidth}
    \includegraphics[width=\linewidth]{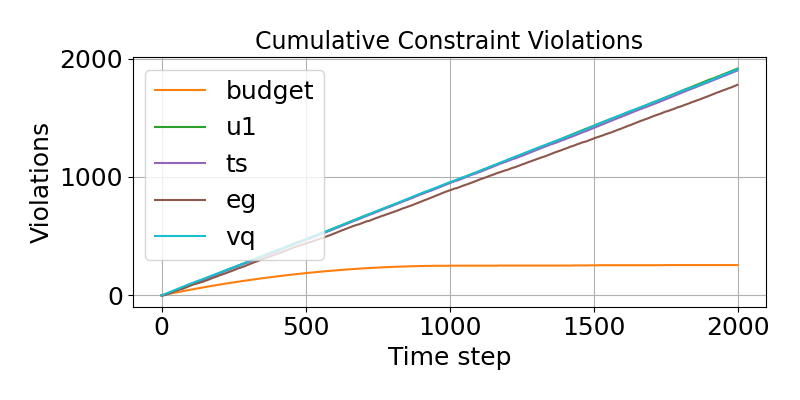}
    \caption{Cumulative constraint violations}
  \end{subfigure}
  \begin{subfigure}{0.32\textwidth}
    \includegraphics[width=\linewidth]{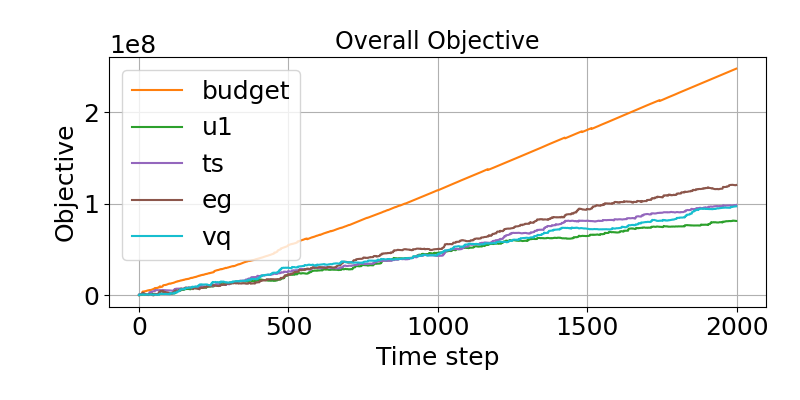}
    \caption{Overall objective ($\Lambda=10^6$)}
  \end{subfigure}
  \begin{subfigure}{0.32\textwidth}
    \includegraphics[width=\linewidth]{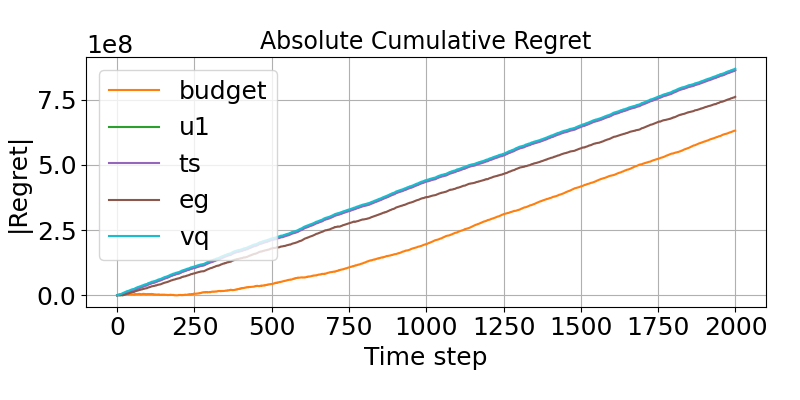}
    \caption{Cumulative absolute regret on throughput}
  \end{subfigure}
  \caption{Performance evaluation under randomly varying energy constraints}
\end{figure*}

\begin{figure*}[!t]
  \centering
  \begin{subfigure}{0.32\textwidth}
    \includegraphics[width=\linewidth]{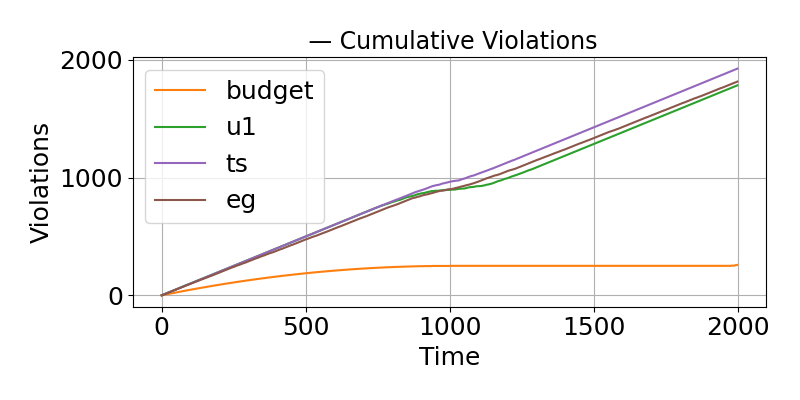}
    \caption{Cumulative constraint violations}
  \end{subfigure}
  \begin{subfigure}{0.32\textwidth}
    \includegraphics[width=\linewidth]{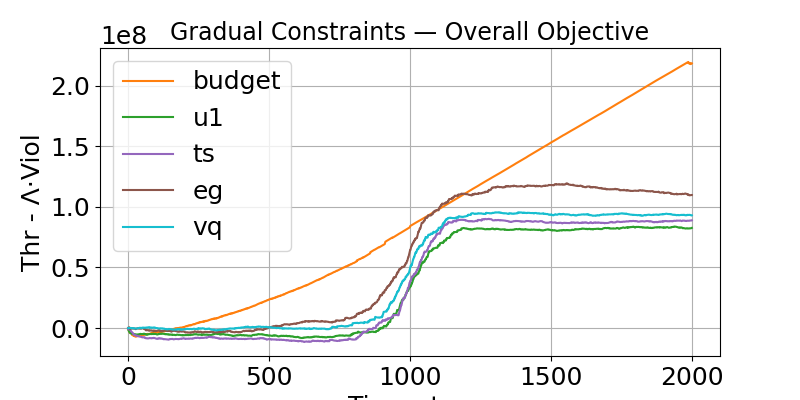}
    \caption{Overall objective ($\Lambda=10^6$)}
  \end{subfigure}
  \begin{subfigure}{0.32\textwidth}
    \includegraphics[width=\linewidth]{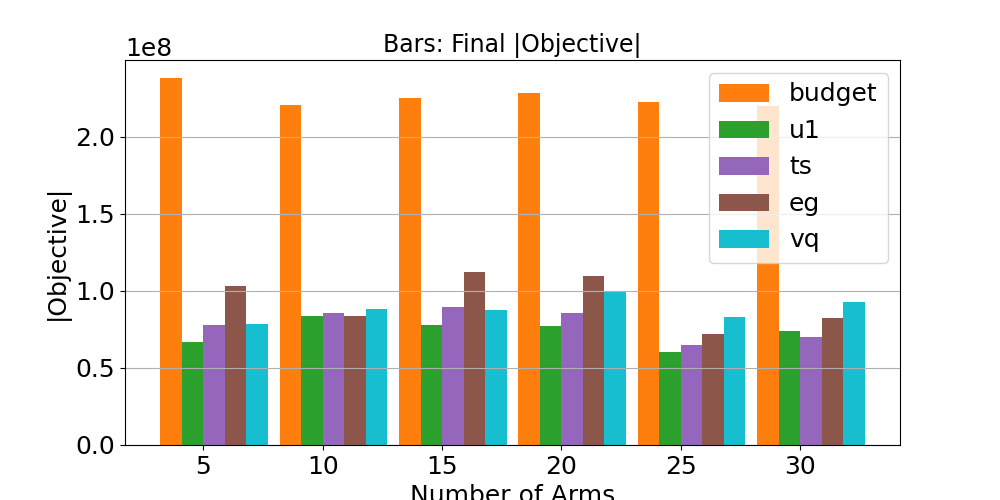}
    \caption{Scalability (Overall Objective)}
  \end{subfigure}
  \caption{Performance evaluation under linearly varying energy constraints}
\end{figure*}


\subsection{Experiment 1: Randomly Varying Energy Constraints}

\subsubsection{Cumulative Constraint Violations}
Under randomly varying energy limits as shown in Figure 2(a), \texttt{Budgeted UCB} (orange) confines its total violations to grow only logarithmically, in accordance with the decaying violation budget.  The three unconstrained baselines—\texttt{u1}, \texttt{ts}, and \texttt{eg}—all rapidly converge to a single high‐throughput arm and thereafter violate the energy cap almost every round.  This illustrates that without any safety mechanism, they pay the full penalty whenever that arm’s energy draw exceeds the instantaneous constraint. Because the \texttt{vq} method doesn’t impose a hard constraint, it keeps choosing high-energy arms even when over budget, so its total violations far exceed those of \texttt{Budgeted UCB}.

\subsubsection{Overall Objective \(\bigl(\Lambda=10^6\bigr)\).}
We plot the overall objective (reward), combining throughput and a penalty \(\Lambda\) on constraint violations:
$\sum_{s=1}^t r_s \;-\;\Lambda \sum_{s=1}^t \mathbf{1}\{c_s > C_s\}.$
As shown in Figure~2(b), \texttt{Budgeted UCB} quickly outpaces all baselines. By strictly limiting violations, it preserves nearly all throughput, so its net reward climbs uninterrupted. The unconstrained algorithms start with high raw throughput but immediately incur massive penalties for breaches; once they identify the top arm, they repeatedly overshoot and depress their net objective. Even occasional later pulls that satisfy the cap never recover to match the budget-aware policy. Meanwhile, \texttt{vq} adjusts only a virtual backlog rather than enforcing a hard safety cutoff, so it continues selecting high-power arms over budget, leading to a steep early spike in violations.

\subsubsection{Cumulative Absolute Throughput Regret.}
Figure 2(c) shows $\bigl|\sum_{s=1}^t \bigl(\mu^\star_s - r_s\bigr)\bigr|,$
the absolute deviation of the throughput from the clairvoyant constrained optimum.  \texttt{Budgeted UCB} achieves a moderate, sublinear increase: after an initial exploration period and occasional remapping around the violation threshold it rapidly settles on the best feasible arm.  The baselines incur larger regret growth—both due to more aggressive exploration (large one-step errors) and because repeated energy cap violations force them away from the true constrained optimum.

\subsection{Experiment 2: Linearly Varying Energy Constraints}

\subsubsection{Cumulative Constraint Violations.}
As shown in Figure 3(a), when the energy cap drifts linearly down and then back up, \texttt{Budgeted UCB} again tracks the shrinking allowance \(\delta_t\), producing a gently rising curve that plateaus during the tightest segment.  The other baselines violate on nearly all timesteps, regardless of the slowly changing threshold. Their violation curves remain essentially straight lines, showing that without appropriate constraint handling for decaying budgets, their performance is less efficient.

\subsubsection{Overall Objective \(\bigl(\Lambda=10^6\bigr)\)}
During the constraint tightening phase, as shown in Figure 3(b), the total rewards of all methods slowly climb - or even drop - because violations are expensive.  \texttt{Budgeted UCB} nonetheless secures a steadily positive slope by trading a small instantaneous throughput loss for far fewer penalties.  Once the cap relaxes, its net reward sharply accelerates as it switches smoothly to higher‐power arms without incurring new violations.  The other baselines exhibit a modest rebound but remain far behind, their early penalties leaving them unable to catch up.

\subsubsection{Scalability Study (Overall Objective)}
The grouped‐bar chart of final overall objective values, plotted against the number of arms \(K\in\{5,10,15,20,25,30\}\), shows that \texttt{Budgeted UCB} (orange) maintains the highest net reward regardless of how many power levels are available, as dipicted in Figure 3(c).  Its bars remain tall and nearly constant as \(K\) grows, demonstrating that the algorithm’s violation‐aware filtering scales gracefully. Even as the action set expands, it quickly zeroes in on the best feasible arm and avoids costly overshoots.  In contrast, the baselines exhibit much lower net objective values that barely improve (and in some cases slightly degrade) when \(K\) increases.  This flat or falling trend reflects that, without appropriate constraint handling for decaying budgets, adding more arms simply prolongs their exploration of high‐power options and incurs even more penalties, so they cannot realize additional throughput gains at scale.  

\section{Conclusions}

In this work, we introduced a novel stochastic bandit model
designed for dynamic and resource-sensitive environments under uncertain characteristic of modern IoT systems. By proposing the \texttt{Budgeted UCB} algorithm, which incorporates a decaying violation budget, we enable principled trade-offs between reward optimization and constraint satisfaction over time. Our theoretical analysis guarantees sublinear regret and logarithmic constraint violations, and our extensive simulations validate that \texttt{Budgeted UCB} significantly outperforms classical baselines in terms of adaptation speed, reliability, and scalability under diverse energy constraint variations. These results demonstrate that the proposed approach effectively balances exploration and safety in evolving environments. The framework opens up promising avenues for future work, including extensions to non-stationary environments, multi-agent settings, and integration with deep learning architectures for complex, high-dimensional IoT applications.


\bibliographystyle{IEEEtran}
\bibliography{references}

\begin{thebibliography}{10}
\providecommand{\url}[1]{#1}
\csname url@samestyle\endcsname
\providecommand{\newblock}{\relax}
\providecommand{\bibinfo}[2]{#2}
\providecommand{\BIBentrySTDinterwordspacing}{\spaceskip=0pt\relax}
\providecommand{\BIBentryALTinterwordstretchfactor}{4}
\providecommand{\BIBentryALTinterwordspacing}{\spaceskip=\fontdimen2\font plus
\BIBentryALTinterwordstretchfactor\fontdimen3\font minus \fontdimen4\font\relax}
\providecommand{\BIBforeignlanguage}[2]{{%
\expandafter\ifx\csname l@#1\endcsname\relax
\typeout{** WARNING: IEEEtran.bst: No hyphenation pattern has been}%
\typeout{** loaded for the language `#1'. Using the pattern for}%
\typeout{** the default language instead.}%
\else
\language=\csname l@#1\endcsname
\fi
#2}}
\providecommand{\BIBdecl}{\relax}
\BIBdecl

\bibitem{xian2025robust}
J.~Xian, J.~Ma, X.~Mei, H.~Wu, N.~Saeed, D.~Han, M.~D. Marino, and K.-C. Li, ``Robust coarse-to-fine 3d-target-localization algorithm for underwater-iot-based networks: Design and performance evaluation under uncertain multi-parameters,'' \emph{IEEE Internet of Things Journal}, 2025.

\bibitem{10742099}
S.~Vaishnav, S.~Khirirat, and S.~Magnússon, ``Communication-adaptive-gradient sparsification for federated learning with error compensation,'' \emph{IEEE Internet of Things Journal}, vol.~12, no.~2, pp. 1137--1152, 2025.

\bibitem{auer2002finite}
P.~Auer, N.~Cesa-Bianchi, and P.~Fischer, ``Finite-time analysis of the multiarmed bandit problem,'' \emph{Machine Learning}, vol.~47, no. 2-3, pp. 235--256, 2002.

\bibitem{badanidiyuru2018bandits}
A.~Badanidiyuru, R.~Kleinberg, and A.~Slivkins, ``Bandits with knapsacks,'' \emph{Journal of the ACM (JACM)}, vol.~65, no.~3, pp. 1--55, 2018.

\bibitem{sui2015safe}
Y.~Sui, A.~Gotovos, J.~Burdick, and A.~Krause, ``Safe exploration for optimization with gaussian processes,'' in \emph{International Conference on Machine Learning}.\hskip 1em plus 0.5em minus 0.4em\relax PMLR, 2015, pp. 997--1005.

\bibitem{moradipari2020stage}
A.~Moradipari, C.~Thrampoulidis, and M.~Alizadeh, ``Stage-wise conservative linear bandits,'' \emph{Advances in neural information processing systems}, vol.~33, pp. 11\,191--11\,201, 2020.

\bibitem{drugan2013designing}
M.~M. Drugan and A.~Now{\'e}, ``Designing multi-objective multi-armed bandits algorithms: A study,'' in \emph{International Joint Conference on Neural Networks (IJCNN)}.\hskip 1em plus 0.5em minus 0.4em\relax IEEE, 2013, pp. 1--8.

\bibitem{wang2025neural}
S.~Wang, S.~Bian, X.~Liu, and Z.~Shao, ``Neural constrained combinatorial bandits,'' \emph{IEEE Transactions on Networking}, 2025.

\bibitem{chen2018bandit}
T.~Chen and G.~B. Giannakis, ``Bandit convex optimization for scalable and dynamic iot management,'' \emph{IEEE Internet of Things Journal}, vol.~6, no.~1, pp. 1276--1286, 2018.

\bibitem{10228901}
Q.~Liu and Z.~Fang, ``Learning to schedule tasks with deadline and throughput constraints,'' in \emph{IEEE INFOCOM 2023 - IEEE Conference on Computer Communications}, 2023, pp. 1--10.

\bibitem{neely2017online}
M.~J. Neely and H.~Yu, ``Online convex optimization with time-varying constraints,'' \emph{arXiv preprint arXiv:1702.04783}, 2017.

\bibitem{cao2018time}
X.~Cao and K.~R. Liu, ``On the time-varying constraints and bandit feedback of online convex optimization,'' in \emph{2018 IEEE International Conference on Communications (ICC)}.\hskip 1em plus 0.5em minus 0.4em\relax IEEE, 2018, pp. 1--6.

\bibitem{cao2018online}
------, ``Online convex optimization with time-varying constraints and bandit feedback,'' \emph{IEEE Transactions on automatic control}, vol.~64, no.~7, pp. 2665--2680, 2018.

\bibitem{Vaishnav2024}
\BIBentryALTinterwordspacing
S.~Vaishnav and S.~Magn{\'u}sson, \emph{Multi-Objective and Constrained Reinforcement Learning for IoT}.\hskip 1em plus 0.5em minus 0.4em\relax Cham: Springer Nature Switzerland, 2024, pp. 153--170. [Online]. Available: \url{https://doi.org/10.1007/978-3-031-50514-0_8}
\BIBentrySTDinterwordspacing

\end{thebibliography}

\end{document}